\documentclass[runningheads]{llncs}
\usepackage{etoolbox} % for "\patchcmd" macro
\makeatletter
\patchcmd{\ps@headings}{\rlap{\thepage}}{}{}{}
\patchcmd{\ps@headings}{\llap{\thepage}}{}{}{}
\makeatother
\usepackage[T1]{fontenc}
\usepackage{amsmath}
\usepackage{multicol}
\usepackage{amsfonts}
\usepackage{xcolor}
\usepackage{graphicx}
\usepackage{tcolorbox}
 \usepackage{booktabs}

\usepackage{tcolorbox}
\usepackage{tikz}
\usepackage[linesnumbered, ruled, vlined]{algorithm2e}
\usepackage{caption}
\usepackage{subcaption}
\usepackage{cite}

\usepackage{url}
\usepackage{hyperref}
\hypersetup{
    colorlinks=true,
    linkcolor=blue,
    filecolor=magenta,      
    urlcolor=cyan,
}
\begin{document}
\title{Scalable Submodular Policy Optimization via Pruned Submodularity Graph}
% \title{Fairness Driven Slot Allocation Problem in Billboard Advertisement \thanks{This work is supported by the Start-Up Research Grant provided by the Indian Institute of Technology Jammu, India (Grant No.: SG100047).}}
\author{Aditi Anand \and Suman Banerjee \and Dildar Ali }
\authorrunning{Anand et al.} % abbreviated author list (for running head)
\institute{Indian Institute of Technology Jammu,
J \& K-181221, India. \\
\email{ \{2022ucs0076,suman.banerjee,2021rcs2009\} @iitjammu.ac.in}}
\maketitle

\begin{abstract}
 In Reinforcement Learning (abbreviated as RL), an agent interacts with the environment via a set of possible actions, and a reward is generated from some unknown distribution. The task here is to find an optimal set of actions such that the reward after a certain time step gets maximized. In a traditional setup, the reward function in an RL Problem is considered additive. However, in reality, there exist many problems, including path planning, coverage control, etc., the reward function follows the diminishing return, which can be modeled as a submodular function. In this paper, we study a variant of the RL Problem where the reward function is submodular, and our objective is to find an optimal policy such that this reward function gets maximized. We have proposed a pruned submodularity graph-based approach that provides a provably approximate solution in a feasible computation time. The proposed approach has been analyzed to understand its time and space requirements as well as a performance guarantee. We have experimented with a benchmark agent-environment setup, which has been used for similar previous studies, and the results are reported. From the results, we observe that the policy obtained by our proposed approach leads to more reward than the baseline methods.    
\end{abstract}
\section{Introduction}
In a \emph{reinforcement learning} problem, an agent interacts with the environment and learns a policy such that the reward at the end of a finite number of steps gets maximized \cite{dong2020deep}. Typically, the environment is modeled as a Markov decision process, which can be expressed as a tuple of six entities: the set of states, the set of actions, the transition function, the reward function, the initial state distribution, and the discount factor \cite{sutton2018reinforcement}. It has been assumed that the reward function is additive and independent of the trajectory. Bellman's optimality principle allows us to identify an optimal policy in polynomial time. In practice, there are many problems whose reward function can not be modeled as an additive function. Hence, our goal is to expand the horizon of RL by extending the reward function to a more general class of functions (other than the additive function). In this paper, we study the RL in a setup with a submodular reward function. In practice, there exist several problems where this setup fits, such as coverage control \cite{prajapat2022near}, experiment design \cite{krause2008near}, entropy maximization \cite{sharma2015greedy}, informative path planning \cite{gunawan2016orienteering}, and many more. To the best of our knowledge, other than the study by Prajapat et al. \cite{prajapat2023submodular}, there does not exist any literature that focuses on the Submodular RL Problem. Given a submodular set function $f$ defined on the ground set $\mathcal{V}$, a well-studied problem is how can we maximize $f$ under different constraints including cardinality constraint \cite{horel2016maximization}, knapsack constraint \cite{yoshida2019maximizing}, matroid constraint \cite{lee2010maximizing}, $p$-system constraint, and many more. It has been reported in the literature that many real-life problems including sensor placement \cite{li2023submodularity}, viral marketing \cite{kempe2003maximizing}, image segmentation \cite{boykov2001interactive}, document summarization \cite{lin2011class}, speeding up of satisfiability solvers \cite{streeter2008combining} etc. can be modeled as maximization of submodular function subject to different constraints. There has been a significant study on this problem in the past four decades, and it has been studied in different setups, including streaming \cite{chakrabarti2015submodular}, distributed \cite{mirzasoleiman2016distributed}, multi-core architecture \cite{honysz2021gpu}, and many more. This leads to a vast amount of literature on this problem, as referenced in the survey on this topic \cite{dughmi2009submodular,krause2014submodular}.

\par While submodularity is typically considered in discrete domains, this concept is also generalized to the continuous domain, such as DR-submodularity by Bian et al. \cite{bian2017continuous}. In this direction, Dolhansky and  Blimes \cite{dolhansky2016deep} introduce the class of deep submodular functions and their neural network models guaranteed to yield to submodular functions as input. This kind of function is useful when the learning is from unknown rewards using function approximation. Among many existing methodologies, one of the methods for solving submodular function maximization is as follows: Starting with an empty set, in each iteration, pick the element that does not belong to the set of selected elements and produces the maximum marginal gain. This process continues until some termination criteria are met. Though this method is easy to understand and easy to implement, it can not be used in practice due to the huge computational time requirement. The primary reason behind this is the large number of marginal gain computations. Several methods have been proposed that try to reduce the number of marginal gain computations \cite{DBLP:conf/kdd/LeskovecKGFVG07}. One such method is the pruned submodularity graph-based approach proposed by Zhou et al. \cite{zhou2017scaling}. Since then, this approach has been used to solve many problems where the input data size is huge. 

\paragraph{\textbf{Motivation.}}
Traditional reinforcement learning (RL) framework predominantly assume that the reward function is additive and trajectory-independent, which simplifies the optimization process using tools like Bellman's principle \cite{dong2020deep,sutton2018reinforcement}. However, this assumption does not hold in many real-world scenarios, such as coverage control \cite{sun2017submodularity}, informative path planning \cite{singh2006efficient}, and experiment design \cite{krause2012near}, where the law of diminishing returns plays a critical role. These applications naturally exhibit submodular reward structures, which introduce a fundamentally different optimization landscape. Despite the wide applicability of submodular functions, very limited attention has been given to reinforcement learning with submodular rewards. This gap motivates the need to extend RL methods to better capture such non-additive, diminishing-return scenarios. Our work addresses this challenge by introducing a scalable approach that leverages pruned submodularity graphs to efficiently approximate optimal policies in submodular RL setups, enabling practical deployment in high-dimensional environments.

\paragraph{\textbf{Our Contribution.}}In this paper, we have made an attempt to use the pruned submodularity graph-based approach to reduce the number of state spaces. In particular, we make the following contributions in this paper:

\begin{itemize}
\item We study the Submodular RL Problem, for which the literature is limited.
\item We have used the pruned submodularity graph-based approach to reduce the state space in a  given Submodular RL Problem.
\item The proposed approach has been analyzed to understand its time and space requirements as well as performance guarantee.
\item We have done a number of experiments with environments that have been used by previous studies, and we observe that our proposed approach leads to more rewards compared to existing methods.
\end{itemize}

\par The Rest of the paper has been organized as follows. Section \ref{Sec:Prob} describes some background concepts and defines our problem formally. Section \ref{Sec:Solution} contains our proposed solution approach with detailed analysis. The experimental evaluation of the proposed approach has been described in Section \ref{Sec:Exp}. Finally, Section \ref{Sec:Conclusion} concludes our study and gives future research directions.

\section{Background and Problem Definition} \label{Sec:Prob}
In this section, we describe the background information related to our problem and define it formally. Initially, we start by describing the submodular function.
\paragraph{\textbf{Set Function and Submodularity.}} 
Consider an $n$ element set $\mathcal{V}=\{v_1, v_2, \ldots, \\ v_n\}$, and a set function $f$ defined on the ground set $\mathcal{V}$. $f$ is said to be non-negative if for any $\mathcal{S} \subseteq \mathcal{V}$, $f(\mathcal{S}) \geq 0$.  For any given subset $\mathcal{S} \subseteq \mathcal{V}$ and an element $v \in \mathcal{V} \setminus \mathcal{S}$, the \emph{marginal gain} of $v$ with respect to $\mathcal{S}$ is denoted by $\Delta(v| \mathcal{S})$ and defined as $\Delta(v| \mathcal{S})= f(\mathcal{S} \cup \{v\})- f(\mathcal{S})$. $f$ is said to be monotone if for all $\mathcal{S} \subseteq \mathcal{V}$ and for all $v \in \mathcal{V} \setminus \mathcal{S}$, $\Delta(v| \mathcal{S}) \geq 0$. $f$ is said to be submodular if for all $\mathcal{A} \subseteq \mathcal{B} \subseteq \mathcal{V}$ and for all $v \in \mathcal{V} \setminus \mathcal{B}$, $\Delta(v| \mathcal{A}) \geq \Delta(v| \mathcal{B})$. The submodularity property of a set function says that the marginal gain of an element decreases as the set size in which the element is added increases, and this property is called the \emph{law of diminishing returns}. One well-studied problem in this context is that given a submodular function and a positive integer $k$, which $k$ nodes should be chosen such that the function becomes maximized? This problem has been referred to as the Submodular Function Maximization Problem. It has been shown in the literature that this problem is \textsf{NP-hard}. However, an incremental greedy approach that works based on marginal gain computation generates a solution which is $(1-\frac{1}{e})$-factor approximate solution to the optimal solution \cite{nemhauser1978analysis}.
\paragraph{\textbf{Markov Decision Process.}} In general, a Reinforcement Learning Problem is posed as a discrete-time stochastic process where an agent interacts with the environment. In each time step, an agent will choose one of the available actions, and depending upon which action it chooses, it goes to a state and draws a reward from an unknown distribution. The task here is to derive the action sequence that the agent should choose so that the cumulative reward after a certain time step will be maximized. A Reinforcement Learning Setup is often modeled as a Markov Decision Process, which can be described using $5$ Tuples $(\mathcal{S}, \mathcal{A}, \mathcal{R}, T, \gamma)$. Here, $\mathcal{S}$ and $\mathcal{A}$ denote the set of states and the set of actions, respectively. $T$ is state transition function, i.e., $T: \mathcal{S} \times \mathcal{A} \times \mathcal{S} \longrightarrow [0,1]$. $\mathcal{R}$ is the reward function, i.e., $\mathcal{R}:\mathcal{S} \times \mathcal{A} \times \mathcal{S} \longrightarrow \mathbb{R}$. $\gamma$ is called as the discount factor. Traditionally, it has been assumed that this process follows the Markovian Property. This means that the next state is decided by the current state and the action that the agent chooses, not the past history of the states in which the agent was in. However, in practice, there are several problems, which include path planning, coverage control, etc., following the reward function, which follows the diminishing return property. Such scenarios can be modeled as a Submodular Reinforcement Learning Problem, which has been stated subsequently.

\paragraph{\textbf{Submodular Reinforcement Learning Set Up.}}
A submodular reinforcement learning problem can be described using $6$ tuples, $\Gamma=(\mathcal{S}, \mathcal{A}, T, H, \rho, \mathcal{R})$. Here, $H$ signifies the finite long-time horizon. Instead of states, here we consider time-augmented states. Hence, the state set $\mathcal{S}= \mathcal{V} \times H $, and an arbitrary state is denoted by $s=(v,h)$. $\mathcal{R}: 2^{\mathcal{S}} \longrightarrow \mathbb{R}$ is submodular reward function where $\mathbb{R}$ denotes the set of real numbers. The transition probability from the state $v$ to $v^{'}$ if the agent chooses an action $a \in \mathcal{A}$ at time step $h \in H$ is denoted as $T((v^{'},h+1)|(v,h),a)$. This process starts at some state $s_0$, and at each time step $h \geq 0$ and at a state $(v,h)$, the agent draws an action $a_h$ according to some predefined policy and reaches the state $s_{h+1}$. The realization of the whole process can be thought of as a trajectory $\tau = \{(s_h,a_h)_{h=0}^{H-1}, s_{H}$\} where $H$ is the length of the episode. Any partial trajectory (from time step $\ell$ to $\ell^{'}$) can be captured as $\tau_{\ell:\ell^{'}}= \{(s_h,a_h)_{h=\ell}^{\ell^{'}-1}, s_{\ell^{'}}$\}. By this notation, $\tau$ can be written as $\tau_{0:H}$. The reward function $\mathcal{R}$ can be calculated after the end of any (partial) trajectory. For the trajectory $\tau_{\ell:\ell^{'}}$, the value of the reward function is denoted by $\mathcal{R}(\tau_{\ell:\ell^{'}})$. In the Markov decision process, an agent acts as per \emph{policy}. A policy is considered Markovian if its actions only depend on the current state. This signifies for any $h \in H$, $\pi(a_h|\tau_{0:h-1})= \pi(a_h|s_h)$. In this study, we relax this condition and assume that the action depends on the previous $k$ states for some given positive integer $k$. For any $h \in H$, this can be given as $\pi(a_h|\tau_{0:h-1})= \pi(a_h|\tau_{h:h-k})$. The set of all Markovian policies and all $ k$-step non-Markovian policies for an MDP is denoted by $\Pi$ and $\Pi^{k}_{NM}$, respectively. Subsequently, we proceed to describe our problem statement.

\paragraph{\textbf{Problem Definition.}} In this paper, given a Submodular MDP, we study the problem of finding an optimal policy to maximize the expected reward. For any given policy $\pi$, let $f(\tau,\pi)$ denote the probability distribution over any random trajectory $\tau$, which can be given by Equation \ref{Eq:1}.

\begin{equation} \label{Eq:1}
f(\tau, \pi)=\rho(s_0) \cdot \prod_{h=0}^{H-1} \pi(a_h|\tau_{0:H}) \cdot T(s_{h+1}, s_h|a_h)
\end{equation}

Now, we define the objective function as the expectation of the trajectory distribution induced by the policy. Our goal here is to find the policy (among a set of given policies) so that the objective function gets maximized.

%\begin{tcolorbox}
%\underline{\textsc{Tag Allocation Problem}} \\
%%\vspace*{2 cm} 
%\textbf{Input:} A trajectory ($\mathcal{D}$) and Billboard ($\mathcal{B}$) Database, A set of slots ($\mathcal{S}$) and tags ($\tau$).
%
%\textbf{Problem:} Find out an allocation of the tags to the slots such that the influence is maximized.
%
%\textbf{Task:} Return the assignment of tags to slots such that the influence is maximized.
%\end{tcolorbox}

\section{Proposed Solution Approaches} \label{Sec:Solution}
In this section, we will discuss the proposed solution approach. First, we start with some preliminary concepts used in Algorithm \ref{Algo: SGPO}. Given a submodular MDP $\mathcal{M}(S, A, T, \rho, H, \mathcal{R})$, we parameterize the policy $\pi$ using neural networks and represent it as $\pi_{\theta}$, where $\theta$ denotes the policy parameters. Next, we define the Expected Return in Definition \ref{Def:ER}.
\begin{definition}[Expected Return]\label{Def:ER}
The performance measure of a policy $\pi_{\theta}$ is its expected return over the trajectory distribution induced by the policy. The expected return is denoted by $J(\pi_{\theta})$ and is given by the following equation:
\begin{equation} \label{Eq:Expected_Return}
    J(\pi_{\theta}) = \mathbb{E}_{\tau \sim f(\tau; \pi_{\theta})} \left[ \mathcal{R}(\tau) \right].
\end{equation}
\end{definition}

Our goal is to find a policy that maximizes the performance measure in the set of all Markovian policies ($\Pi$).
\begin{equation}
\pi^* = \arg\max_{\pi \in \Pi} J(\pi)
\end{equation}

\paragraph{\textbf{Approach:}} The policy parameters ($\theta$) are updated using the following regularized gradient ascent equation:
\begin{equation}\label{Eq:4}
    \theta \leftarrow \theta + \arg\max_{\delta\theta: \delta\theta+\theta\in\Theta} \delta\theta^\top \nabla_\theta J(\pi_\theta) - \frac{1}{2\alpha} \|\delta\theta\|^2.
\end{equation}
where $\alpha$ is a regularization coefficient.

\begin{definition}[Submodularity Graph]
The submodularity graph is a weighted directed graph $G(V, E, w)$ defined by a submodular function $F: 2^V \to \mathbb{R}$, where $V$ is the set of states corresponding to the ground set. Each directed edge $e = (u, v) \in E$ from $u$ to $v$ has weight:

\begin{equation}
    w_{uv} = F(v \mid u) - F(u \mid V \setminus \{u\})
\end{equation} 
\end{definition}

\begin{definition}[Divergence of a state]
Given submodular function $F$, divergence of a state $v$ from a set of states $U$ is defined as: 
\begin{equation}
    w_{Uv} = \underset{u \in U}{min} \ [F(v|u)-F(u| V \setminus \{u\})]
\end{equation}
On the submodularity graph $G(V, E, w)$, it is given by:
\begin{equation}
    w_{Uv} = \underset{u \in U}{min} \ w_{uv}
\end{equation}
\end{definition}

$F(v \mid u)$ denotes the maximum possible gain on adding $v$ to a set containing $u$ and $F(u| V \setminus \{u\})$ is the minimum gain on adding $u$ to set $V \setminus \{u\}$. The small value of the first term implies that $v$ is unimportant if $u$ is retained, while a larger value of the second term denotes the importance of $u$. Thus, a smaller divergence for an element $v$ suggests that it can be removed while we retain the set $U$. This helps remove redundant states.

\begin{algorithm}[h]
\SetAlgoLined
\KwData{SMDP $\mathcal{M}(S,A,T, \rho, H, \mathcal{R})$, $\pi, N, r, c$}
\KwResult{ Trained Policy $\pi$}
 \For{$\text{epoch}~ k = 1 ~\text{to}~ N$}{
 Initialize $V \leftarrow \emptyset$, $Z \leftarrow \emptyset$, $R \leftarrow \emptyset$\;
\For{$h = 0~ \text{to}~ H - 1$}{
Sample $a_{h} \sim \pi(a_{h}|s_{h})$, execute $a_{h}$\;
$V \leftarrow V \cup \{s_{h}\}$\;
$Z \leftarrow Z \cup \{a_{h}\}$\;
$R \leftarrow R \cup \mathcal{R}(\tau _{0 : h+1})$\;
}
 Initialize $V^{'} \leftarrow \emptyset$, $U \longleftarrow \emptyset$, $|V|=n$ \;
 $\text{Construct the submodularity graph with the states in }V$\;
 \While{$|V| > \ r \cdot \log n$}{
 $\text{Sample out }r \cdot \log n ~\text{states uniformly at random from }V  \text{ and place them in }U$\; 
 $V \longleftarrow V \setminus U$, $V^{'} \longleftarrow V^{'} \cup U $\;
 \For{$\text{All }v \in V $}{
 $w_{Uv} \longleftarrow \ \underset{u \in U}{min} \ [\mathcal{R}(v|u)-\mathcal{R}(u| V \setminus \{u\})]$\;
 }
 $\text{Remove }(1-\frac{1}{\sqrt{c}}) \cdot |V| \text{ states from }V \text{ having smallest value of }w_{Uv}$\;
 }
 $V^{'} \longleftarrow V \  \cup \ V^{'} $\;
 $\mathcal{K} \leftarrow \{(s_{h},a_{h},\mathcal{R}(\tau_{0:h+1}))| s_{h} \in V^{'}, a_{h} \in Z, \mathcal{R}(\tau_{0:h+1}) \in R \}$\;
 Estimate $\nabla_{\theta} J(\pi_{\theta})$ using Theorem \ref{Th:Theorem2}\;
 Update policy parameters$(\theta)$ using Equation \ref{Eq:4}\;
}
\caption{Submodularity Graph-based Policy Optimization (SGPO)}
 \label{Algo: SGPO}
\end{algorithm}	

\paragraph{\textbf{Description of the Algorithm.}} The proposed solution approach is presented in Algorithm \ref{Algo: SGPO}. This approach takes SMDP $\mathcal{M}(S, A, T, \rho, H, \mathcal{R})$, policy $\pi$, number of epochs $N$ and sparsification parameters $r$ and $c$ as input to our algorithm and returns the optimized trained policy as output. The agent is represented as a stochastic policy parameterized by a neural network. In each training iteration, the agent samples an action from the policy and executes it over each time step (line $4$). The sets $V$, $Z$, and $R$ are used to maintain the trajectory states, actions, and cumulative reward, respectively. Upon completion of the episode in lines $3$ to $7$, we construct a submodular graph with the states in $V$. Subsequently, we apply submodular sparsification to the set of states $V$ and iteratively sample out $r\cdot \log n$ states uniformly at random from $V$ and put them to $V^{'}$ (in \texttt{while loop} at line $12$). In lines $14$ to $16$, the weight of each state is calculated, and the states having smaller weight are pruned, that is, $(1-\frac{1}{\sqrt{c}}) \cdot |V|$ many states from the set $V$. In lines $11$ to $18$, after completion of \texttt{while loop}, Algorithm \ref{Algo: SGPO} obtains a pruned set of states $V^{'}$. Next, in line $20$, the updated trajectory is obtained in $K$. Finally, in lines $21$ and $22$, $\Delta_{\theta} J(\pi_{\theta})$ is estimated using Theorem \ref{Th:Theorem2} and policy parameters $(\theta)$ are updated using Equation \ref{Eq:4}. The values of r and c were set at 8 for the experiments.

\paragraph{\textbf{Hyperparameter Choice and Sensitivity.}}
Algorithm \ref{Algo: SGPO} relies on two key hyperparameters: $r$, which controls the number of states sampled per iteration $( r \cdot \log n )$, and $c$, which determines the fraction of states pruned per iteration $(\left(1 - \frac{1}{\sqrt{c}}\right) \cdot |V|)$. In our experiments, we set $r = 8$ and $c = 8$. These values were chosen on the basis of empirical performance in preliminary experiments, balancing the trade-off between state-space reduction and preservation of high-reward trajectories. Specifically, $r = 8$ ensures a moderate sampling rate, allowing sufficient exploration of the state space while keeping the number of iterations logarithmic $(\mathcal{O}(\log n))$. Similarly, $c = 8$ results in pruning approximately $65\%$ of the remaining states per iteration $(1 - \frac{1}{\sqrt{8}} \approx 0.646)$, which aggressively reduces the state space while retaining states with high divergence weights, as informed by prior work on submodular maximization \cite{ali2022influential,ali2023influential,zhou2017scaling}.

% \begin{table}[h!]
%     \centering
%     \caption{Sensitivity Analysis of Hyperparameters \( r \) and \( c \)}
%     \label{tab:sensitivity}
%     \begin{tabular}{ccccc}
%         \toprule
%         Environment & \( r \) & \( c \) & Objective Function (Mean) & Runtime (s) \\
%         \midrule
%         Car Racing & 4 & 8 & 125.4 & 245.1 \\
%         Car Racing & 8 & 8 & 127.8 & 210.3 \\
%         Car Racing & 12 & 8 & 118.2 & 198.7 \\
%         Car Racing & 8 & 4 & 126.1 & 230.5 \\
%         Car Racing & 8 & 16 & 115.6 & 192.4 \\
%         Graph-Based (M) & 4 & 8 & 128.3 & 180.2 \\
%         Graph-Based (M) & 8 & 8 & 129.8 & 165.7 \\
%         Graph-Based (M) & 12 & 8 & 122.5 & 158.9 \\
%         Graph-Based (M) & 8 & 4 & 127.9 & 175.4 \\
%         Graph-Based (M) & 8 & 16 & 120.1 & 150.6 \\
%         \bottomrule
%     \end{tabular}
% \end{table}

\paragraph{\textbf{Complexity Analysis.}}
Now, we analyze the time and space requirements for Algorithm \ref{Algo: SGPO}. In Line No. $1$, \texttt{For Loop} will iterate for $\mathcal{O}(N)$ time and initializing $V$, $Z$ and $R$ will take $\mathcal{O}(1)$ time. In Line No. $3$ to $8$ the \texttt{For Loop} will execute for $\mathcal{O}(H)$ time and Line No. $9$ will take $\mathcal{O}(1)$ time. Next, to construct the submodularity graph in Line No. $10$, it will take $\mathcal{O}(n^{2} \cdot n)$, that is, $\mathcal{O}(n^{3})$, where the construction of the state pair will take $\mathcal{O}(n^{2})$, and calculating the weight for each edge using the reward function $\mathcal{R}$ will take $\mathcal{O}(n)$. The \texttt{While Loop} at Line No. $11$ will run for $\mathcal{O}(\log n)$ time, and in Line No. $12$, the random selection of $r \cdot \log n$ many elements will take $\mathcal{O}(\log n)$ time. In Line No. $14$ to $16$ \texttt{For Loop} will take $\mathcal{O}(n^{2} \cdot \log n)$ time and Line No. $17$ will take $\mathcal{O}(n \cdot \log n)$. So, from Line No. $11$ to $18$ will take $\mathcal{O}(\log n \cdot \log n + \log n \cdot n^{2}\log n + \log n \cdot n \log n )$ i.e., $ \mathcal{O}(n^{2} \log^{2} n)$. In Line No. $20$ computing $\mathcal{K}$ will take $\mathcal{O}(\log^{2} n)$ time and in Line No. $21$ will take $\mathcal{O}(H(sm)\cdot H(sm) \cdot H(n)])$ i.e., $\mathcal{O}(n^{4}s^{2}m^{2})$, where $H$ is the horizon, $s$ is the state dimension and $m$ is the action dimension. Finally, updating the policy will take $\mathcal{O}(1)$ time. Hence, Line No. $1$ to $23$ will take $\mathcal{O}(N[H + n^{3} + n^{2} \log^{2}n + n^{4} s^{2} m^{2}])$ i.e., $\mathcal{O}(N \cdot n^{4} s^{2} m^{2})$ time to execute.
\par The additional space requirement by Algorithm \ref{Algo: SGPO}  to store $V$, $Z$, $R$, and $U$ will take $\mathcal{O}(n)$, $\mathcal{O}(n)$, $\mathcal{O}(n)$, and $\mathcal{O}(\log n)$, respectively. To store the edges of the graph, it will take $\mathcal{O}(n^{2})$ time. So, the total space requirement for Algorithm $1$ will be $\mathcal{O}(n^{2})$.

\begin{theorem}\label{Th:Complexity}
The time and space requirements by Algorithm \ref{Algo: SGPO} will be of $\mathcal{O}(N \cdot n^{4} s^{2} m^{2})$ and $\mathcal{O}(n^{2})$, respectively.
\end{theorem}
\begin{theorem}
Let $\text{OPT} = \max_{\pi \in \Pi} J(\pi)$ be the optimal expected submodular return in a Submodular Reinforcement Learning (SubRL) problem. Then, for any constant $\gamma > 0$, there exists no polynomial-time algorithm that guarantees a policy $\pi$ satisfying:
\begin{equation}
    J(\pi) \geq \frac{\text{OPT}}{\log^{1-\gamma} \text{OPT}}
\end{equation}
for all instances of SubRL, unless $\text{NP} \subseteq \text{ZTIME}(n^{\text{polylog}(n)})$. Thus, under standard complexity-theoretic assumptions, the SubRL problem cannot be approximated within a constant factor and is limited to, at best, a logarithmic approximation.
\end{theorem}

\begin{proof}
We establish the inapproximability of Submodular Reinforcement Problem (SubRL) by reducing from Group Steiner Tree (GST) problem \cite{garg2000polylogarithmic}, which is known to be hard to approximate within $O(\log^2 n)$ unless $\text{NP} \subseteq \text{ZTIME}(n^{\text{polylog}(n)})$. 

In the GST problem given a weighted graph $G = (V, E)$ with edge costs $w(e)$, a root node $r \in V$, a collection of groups $G_1, G_2, ..., G_k \subseteq V$, the objective is to find a minimum-cost subtree that contains at least one node from each group $G_i$. Now, the GST problem can be reduced to Submodular Orienteering Problem (SOP) \cite{prajapat2023submodular}, where given a root node $r \in V(T)$. The goal is to find a walk of length at most $B$ that maximizes submodular function $\mathcal{R}$ defined on the nodes of the underlying graph. Now, we construct an instance of SOP from GST, where the graph $G$ remains unchanged, the root node $r$ is the start state and a budget $B$ is assigned based on the optimal cost of the Steiner tree. We define a submodular reward function $\mathcal{R}: 2^V \to \mathbb{R}$ such that $\mathcal{R}(S) = \sum_{i=1}^{k} \min(1, |S \cap G_i|)$, where $S$ is the subset of nodes. The goal in SOP is to find a path of cost at most $B$ that maximizes $\mathcal{R}(S)$. Since SOP is at least as hard as GST, it inherits its inapproximability bounds. We now show that SOP can be reduced to a SubRL problem by constructing an instance of SubRL from SOP. We define a Submodular Markov Decision Process with states as nodes in $G$, actions as available edges, transitions following the edges of $G$, the reward function is the same submodular function $\mathcal{R}$ as in SOP and horizon $H = B$. Finding an optimal policy in SubRL corresponds to solving the SOP. Since SOP is hard to approximate within $\mathcal{O}(\log^{1-\gamma} \text{OPT})$, the same bound applies to SubRL; if SubRL were approximable within a constant factor, then GST would be as well, contradicting known hardness results. Thus, SubRL cannot be approximated better than $O(\log^{1-\gamma} \text{OPT})$ unless $\text{NP} \subseteq \text{ZTIME}(n^{\text{polylog}(n)})$.
\hfill $\square$
\end{proof}

\begin{theorem}\label{Th:Theorem2}
Given a Submodular Markov Decision Process $\mathcal{M}$ with reward function $\mathcal{R}$ and policy parameters $\theta$, gradient of the performance measure can be estimated as
\begin{equation}
    \nabla_{\theta} J(\pi_{\theta}) = \mathbb{E}_{\tau \sim f(\tau; \pi_{\theta})} \left[ \sum_{i=0}^{H-1} \nabla_{\theta} \log \pi_{\theta} (a_i | s_i) (\sum_{j=i}^{H-1} \mathcal{R}(s_{j+1} | \tau_{0:j})+\mathcal{R}(s_0)) \right]
\end{equation}
\end{theorem} 

\begin{proof} The performance measure from Equation \ref{Eq:Expected_Return} we get,
\begin{equation}
    J(\pi_{\theta}) = \mathbb{E}_{\tau \sim f(\tau; \pi_{\theta})} \left[ \mathcal{R}(\tau) \right]
\end{equation}
\begin{equation}
    J(\pi_{\theta}) = \sum_{\tau} f(\tau; \pi_{\theta}) \mathcal{R}(\tau)
\end{equation}

Taking the gradient with respect to $\theta$, we get
\begin{equation}
    \nabla_{\theta} J(\pi_{\theta}) = \sum_{\tau} \nabla_{\theta} f(\tau; \pi_{\theta}) \mathcal{R}(\tau)
\end{equation}

Applying the log-derivative trick, we get
\begin{equation}
    \nabla_{\theta} \log f(\tau; \pi_{\theta}) = \nabla_{\theta} f(\tau; \pi_{\theta}) / f(\tau; \pi_{\theta})
\end{equation}
\begin{equation}
    \nabla_{\theta} f(\tau; \pi_{\theta}) = f(\tau; \pi_{\theta}) \nabla_{\theta} \log f(\tau; \pi_{\theta})
\end{equation}

So,
\begin{equation}
    \nabla_{\theta} J(\pi_{\theta}) = \sum_{\tau} f(\tau; \pi_{\theta}) \nabla_{\theta} \log f(\tau; \pi_{\theta}) \mathcal{R}(\tau)
\end{equation}

Now, taking expectation over trajectories, we get
\begin{equation}\label{Eq:16}
    \nabla_{\theta} J(\pi_{\theta}) = \mathbb{E}_{\tau \sim f(\tau; \pi_{\theta})} \left[ \nabla_{\theta} \log f(\tau; \pi_{\theta}) \mathcal{R}(\tau) \right]
\end{equation}

From Equation \ref{Eq:1} using the trajectory probability $f(\tau; \pi_{\theta})$ we have
\begin{equation}
    f(\tau; \pi_{\theta}) = \rho(s_0) \prod_{i=0}^{H-1} \pi_{\theta}(a_i | s_i) P(s_{i+1} | s_i, a_i)
\end{equation}

Taking the logarithm, we get
\begin{equation}
    \log f(\tau; \pi_{\theta}) = \log \rho(s_0) + \sum_{i=0}^{H-1} \log \pi_{\theta}(a_i | s_i) + \sum_{i=0}^{H-1} \log P(s_{i+1} | s_i, a_i)
\end{equation}

Since the transition probabilities $P(s_{i+1} | s_i, a_i)$ are independent of $\theta$, their gradient vanishes. So we can write
\begin{equation}
    \nabla_{\theta} \log f(\tau; \pi_{\theta}) = \sum_{i=0}^{H-1} \nabla_{\theta} \log \pi_{\theta}(a_i | s_i)
\end{equation}

Thus, substituting back in Equation \ref{Eq:16} we get
\begin{equation}
    \nabla_{\theta} J(\pi_{\theta}) = \mathbb{E}_{\tau \sim f(\tau; \pi_{\theta})} \left[ \sum_{i=0}^{H-1} \nabla_{\theta} \log \pi_{\theta} (a_i | s_i) \mathcal{R}(\tau) \right]
\end{equation}

Using the submodular property of $F$, we expand it in terms of marginal gains:
\begin{equation}
    \mathcal{R}(\tau) = \sum_{j=0}^{H-1} \mathcal{R}(s_{j+1} | \tau_{0:j}) + \mathcal{R}(s_0)
\end{equation}

Thus, we can write
\begin{equation}
    \nabla_{\theta} J(\pi_{\theta}) = \mathbb{E}_{\tau \sim f(\tau; \pi_{\theta})} \left[ \sum_{i=0}^{H-1} \nabla_{\theta} \log \pi_{\theta} (a_i | s_i) (\sum_{j=i}^{H-1} \mathcal{R}(s_{j+1} | \tau_{0:j})+\mathcal{R}(s_0)) \right]
\end{equation}
\hfill $\square$
    
\end{proof}

\section{Experimental Evaluation} \label{Sec:Exp}
In this section, we describe the experimental evaluation of the proposed solution approach. Initially, we start by describing the environments that we have considered in our experiments.
\subsection{Environments}
We have tested our algorithm's performance in various environments involving different state-action spaces and reward functions. The state-action spaces are divided into two categories: continuous and discrete. We have considered two continuous state-action spaces, namely, \emph{Car Racing} and \emph{MuJoCo Ant}, and two discrete state-action spaces, \emph{Graph Based Environment} and \emph{Entropy Based Environment}, and have been described subsequently.

\paragraph{\textbf{Car Racing Environment.}} This is a high-dimensional environment where a race car tries to finish the racing lap as fast as possible. The environment is accompanied by an important challenge of learning a policy to maneuver the car at the limit of handling. The track is challenging, consisting of 13 turns with different curvatures. The car has a six-dimensional state space representing position and velocity. The control commands are two-dimensional, representing throttle and steering. The state representation of an individual car is denoted as $z=[\rho,d, \mu,V_x,V_y,V_z]$ where $\rho$ measures the progress along the reference path, $d$ quantifies the deviation from the reference path and $\mu$ characterizes the local heading relative to the reference path. $V_x$ and $V_y$ represent the longitudinal and lateral velocities in the car’s frame, respectively. $\Psi$ represents the car’s yaw rate. The car’s inputs are represented as $[D, \delta]^{T}$ where $D$ belongs to the interval $[-1,+1]$, representing the duty cycle input to the electric motor, ranging from full braking at $-1$ to full acceleration at $1$. $\delta \in [-1,+1]$ corresponds to the steering angle. The car has a camera and observes a patch around its state $s$ as $D_s$. The objective function for this environment is $T(\tau)= |\underset{s \in \tau}{\bigcup} D_{s}|$, and for our problem context, we fix the length of the time horizon as $700$. 

\paragraph{\textbf{MuJoCo Ant Environment.}} In this environment, the state space dimension is $30$ and contains information about the robot’s pose and the internal actuator’s orientation.  The control input dimension is $8$, consisting of torque commands for each actuator. The Ant at any location $s$ covers locations in $2D$ space, $D_s$, and receives a reward based on it. The goal is to maximize $T(\tau)= |\underset{s \in \tau}{\bigcup} D_{s}|$. For detailed information about the Mujoco Ant environment, we refer the reader to \url{https://gymnasium.farama.org/environments/mujoco/ant/.}

\paragraph{\textbf{Graph Based Environment.}} In this environment, consisting of a set of nodes, each associated with a weight, the agent interacts with this space through sequential decision-making. There are two environmental operational modes: $M$ Mode and $SRL$ Mode. In M mode, the agent is required to maximize the total sum of weights for all uniquely covered nodes. Also, the agent is required to cover as many nodes as possible while maximizing the total weight. In SRL Mode, the agent aims to maximize the total sum of weights associated with all covered nodes, allowing for repeated coverage of previously visited nodes to contribute to the total reward.

\paragraph{\textbf{Entropy Based Environment.}} The entropy-based environment is designed to regulate the agent's exploration and exploitation behavior by utilizing entropy as an objective. It operates in two distinct modes: M Mode and SRL Mode. In M Mode, the agent aims to minimize entropy, encouraging more deterministic behavior by reducing the uncertainty in its trajectory. This results in a more structured and predictable exploration pattern. Conversely, in SRL Mode, the objective is to maximize entropy, which promotes exploration by maintaining uncertainty in the agent’s trajectory. This ensures the agent does not prematurely converge to a specific policy, thereby enhancing its ability to discover new or underexplored regions within the environment. The entropy-based approach provides a principled mechanism for balancing exploration and exploitation in reinforcement learning settings.

\begin{table}[htbp]
    \centering
    \caption{Reward functions for different environments}
    \label{tab:reward_functions}
    \resizebox{\textwidth}{!}{
    \begin{tabular}{lccccl}
        \toprule
        \textbf{Environment} & \textbf{Cont/Discrete} & \textbf{State Dim} & \textbf{Action Dim} & \textbf{Function} \\
        \midrule
        Car Racing & Continuous & 6 & 2 & \( T(\tau) = |\bigcup_{s \in \tau} D_s| \) \\
        MuJoCo Ant & Continuous & 30 & 8 & \( T(\tau) = |\bigcup_{s \in \tau} D_s| - \lambda \sum_{s \in \tau}\mathbf{1}(N_s > 1) \) \\
        GP (M) & Discrete & 2 & 5 & \( T(\tau) = \sum_{s \in \tau} w_s - \lambda \sum_{s \in \tau} N_s \) \\
        GP (SRL) & Discrete & 2 & 5 & \( T(\tau) = \sum_{s \in \tau} w_s \) \\
        Entropy (M) & Discrete & 2 & 5 & \( T_M(\tau) = \frac{1}{2}\log\det(\Sigma_{\tau}+10^{-6}I)+9.965784\cdot|\tau| \) \\
        Entropy (SRL) & Discrete & 2 & 5 & \( T_{SRL}(\tau) = \sum_{s \in \tau}\frac{1}{2}\log_2(\Sigma_{ss}+10^{-6})+9.965784 \) \\
        \bottomrule
    \end{tabular}}
\end{table}

\paragraph{\textbf{Experimental Setup.}}
All the Python codes are executed in HP Z2 workstations with 32 GB RAM and an Xeon(R) 3.0 GHz processor. We have compared our proposed SGPO approach with the existing SubPO \cite{prajapat2023submodular} approach. 

\subsection{Goals of the Study}\label{Sec:Goals}
In this study, we will address the following Research Questions (RQ).
\begin{itemize}
    \item \textbf{RQ1:} How do critic loss and policy loss vary with the number of epochs?
    \item \textbf{RQ2:} Varying number of epochs, how do the SGPO and SubPO perform with respect to the number of steps, coverage, and weight value?
    \item \textbf{RQ3:} How does the advantage vary while varying the number of epochs?
\end{itemize}
\subsection{Results and Discussion} 
In this section, we analyze the experimental results and address all the research questions that arise in Section \ref{Sec:Goals}. We subsequently present the results in tabular form for the discrete environments. All the objective functions are summarized in Table \ref{tab:reward_functions}. We experimented one hundred times independently for each environment and reported the objective function's mean value in Table \ref{tab:comparison}.
% \vspace{-0.25cm}
\begin{figure*}[htbp]
\centering
\begin{tabular}{ccc}
\includegraphics[scale = 0.16]{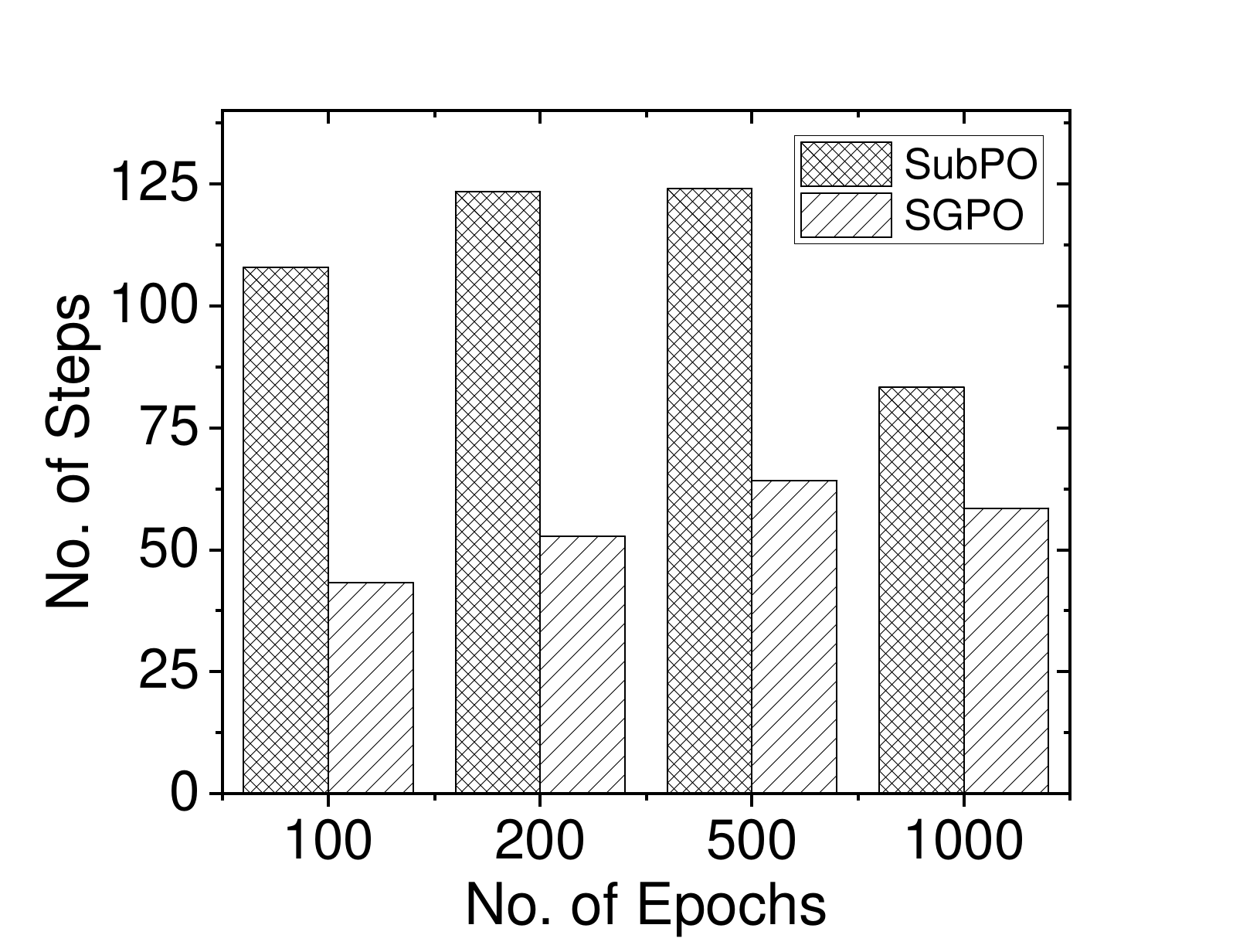}\hspace{-1em} & \includegraphics[scale = 0.16]{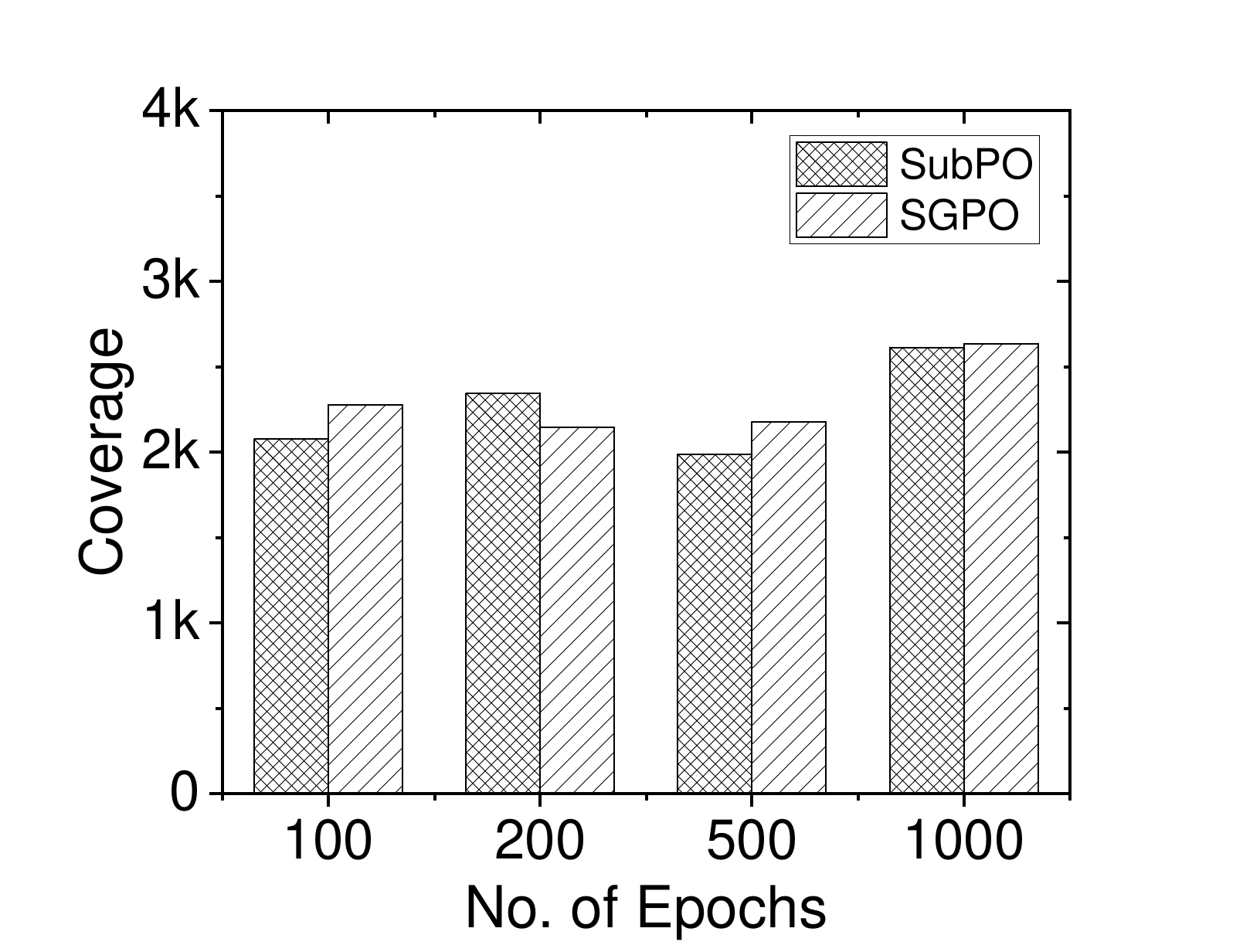}\hspace{-1em} & \includegraphics[scale = 0.16]{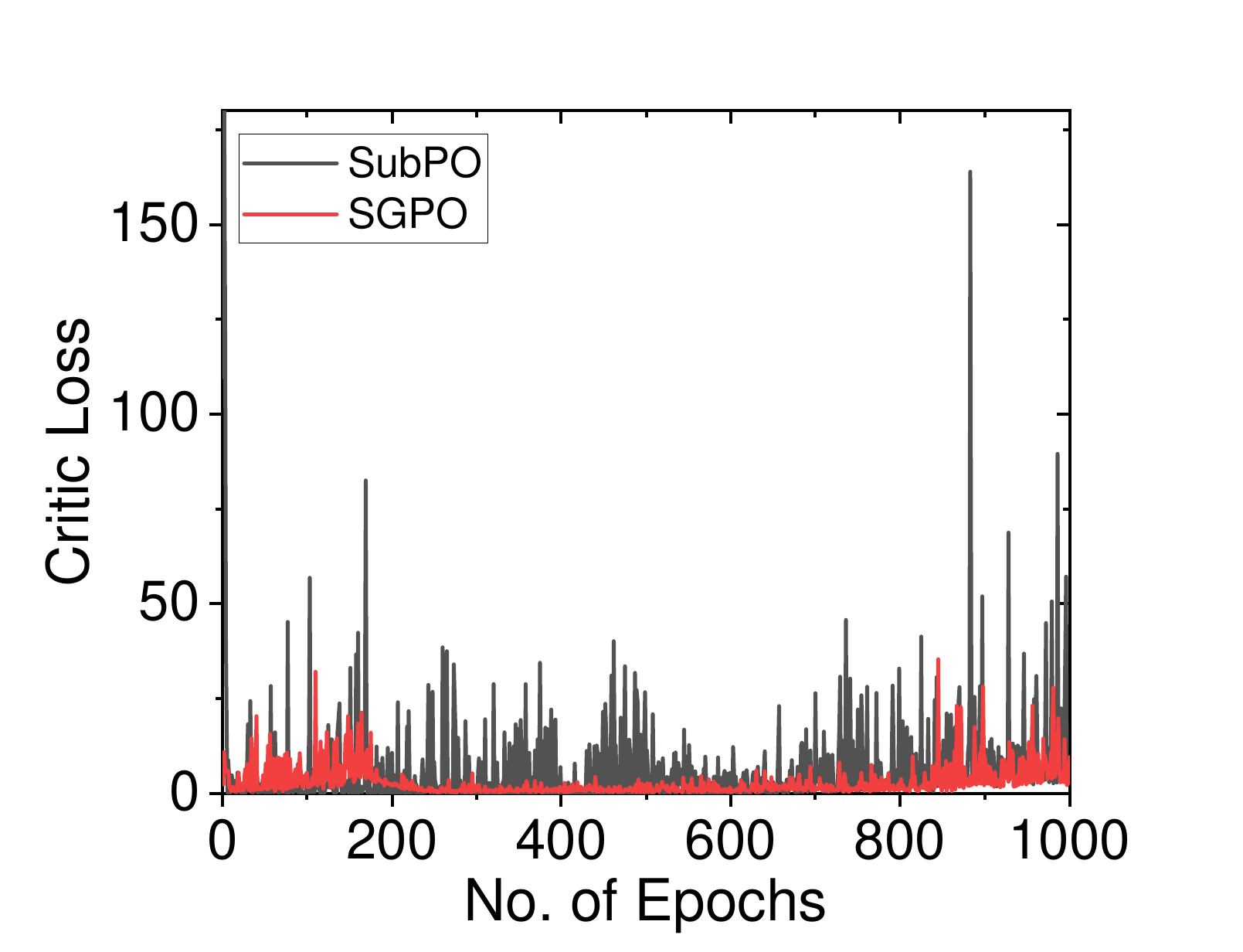} \\
\tiny{(a) Epochs Vs. Steps} &  \tiny{(b) Epochs Vs Coverage} & \tiny{(c) Epochs Vs. Critic Loss}  \\
\includegraphics[scale = 0.16]{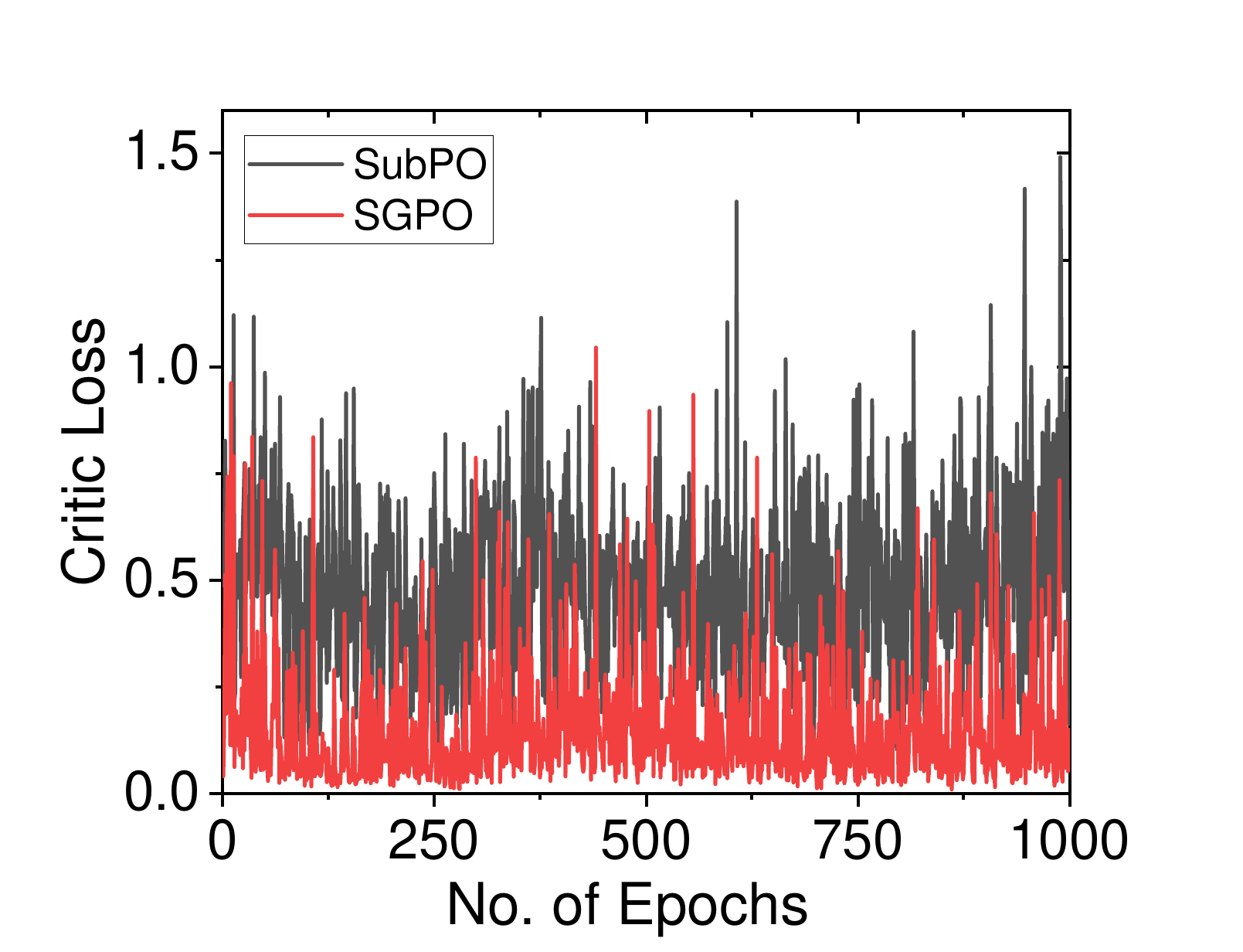}\hspace{-1em} & \includegraphics[scale = 0.16]{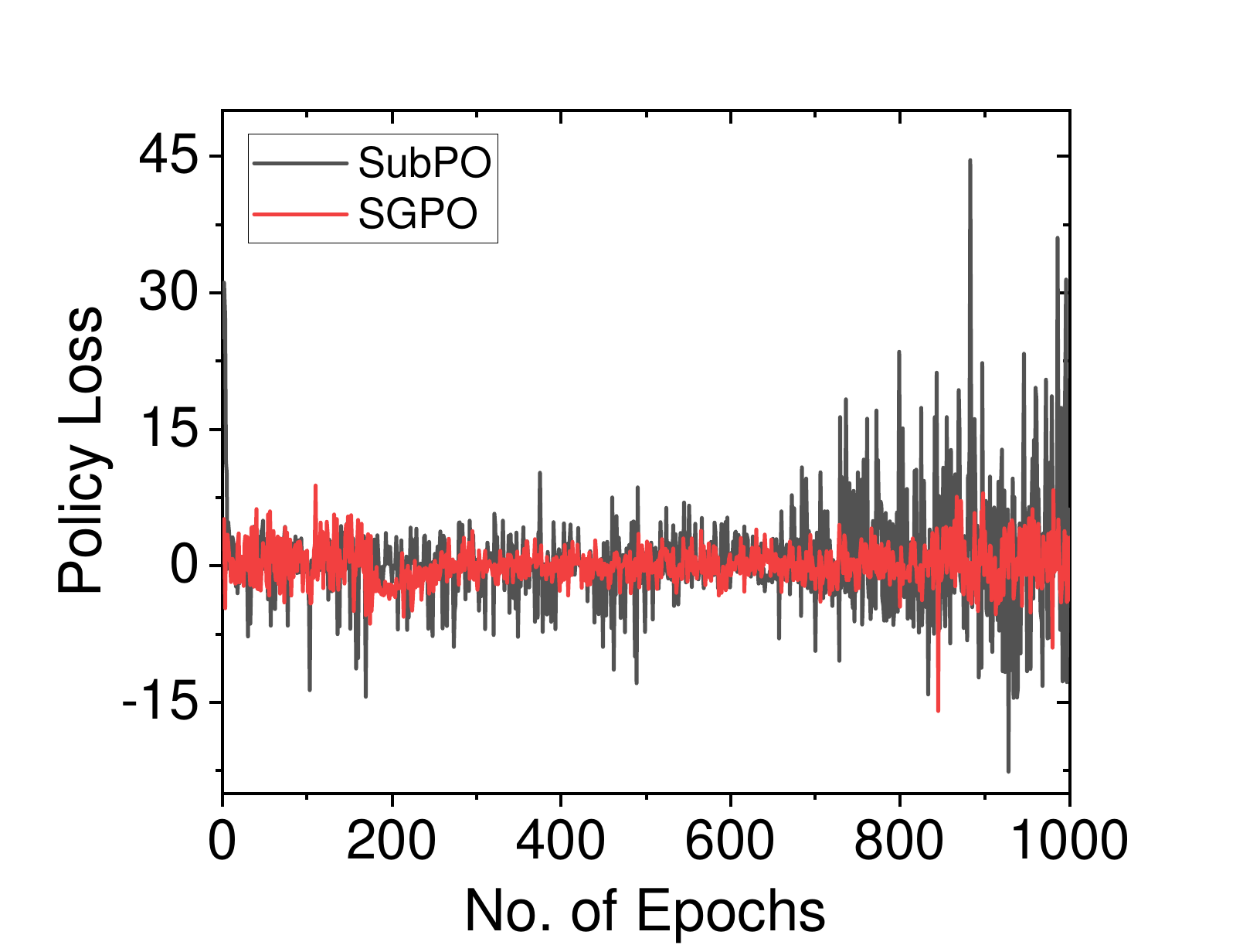}\hspace{-1em} & \includegraphics[scale = 0.16]{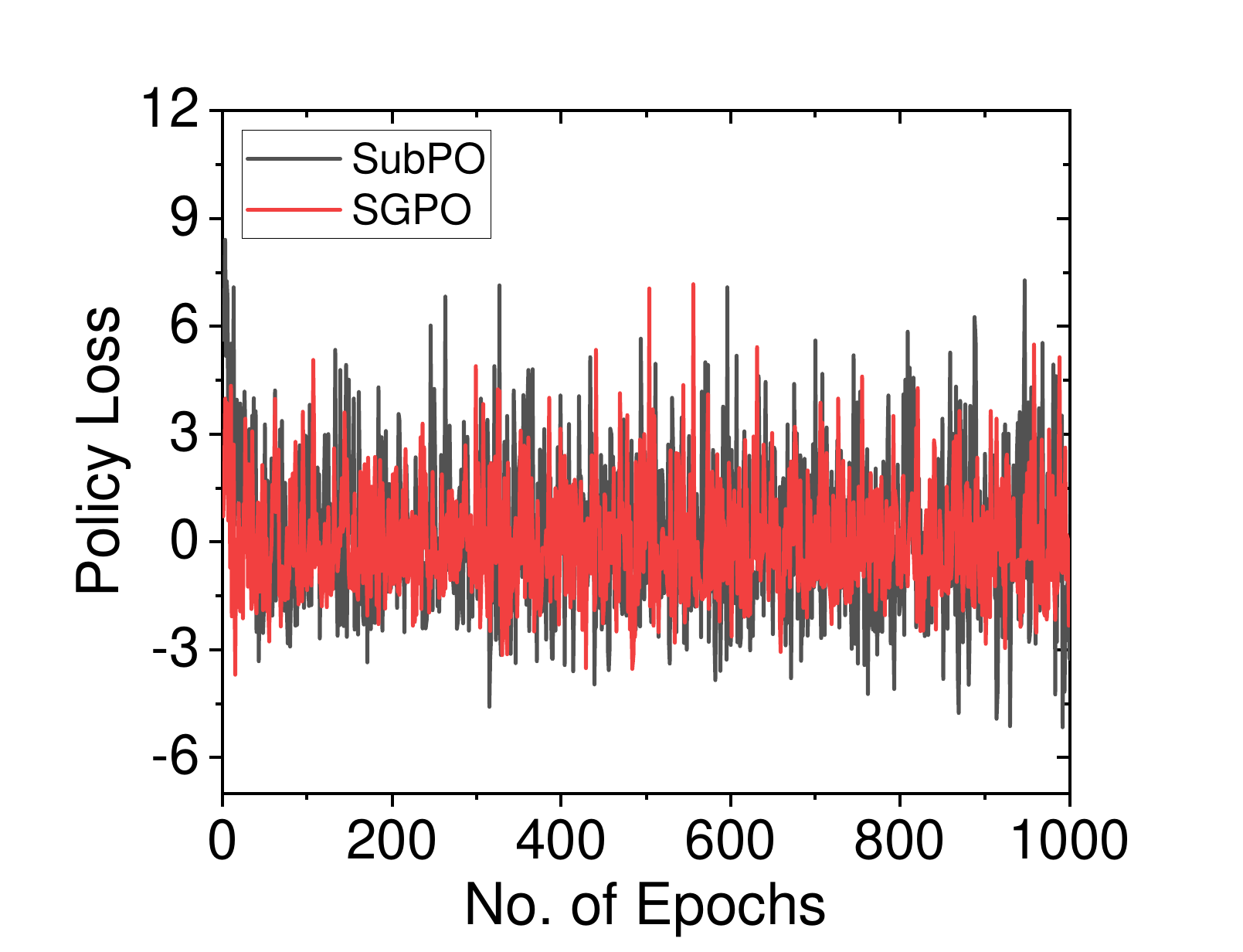}  \\
\tiny{(d) Epoch Vs. Critic Loss} & \tiny{(e) Epochs Vs. Policy Loss}  & \tiny{(f) Epochs Vs. Policy Loss} \\

\includegraphics[scale = 0.16]{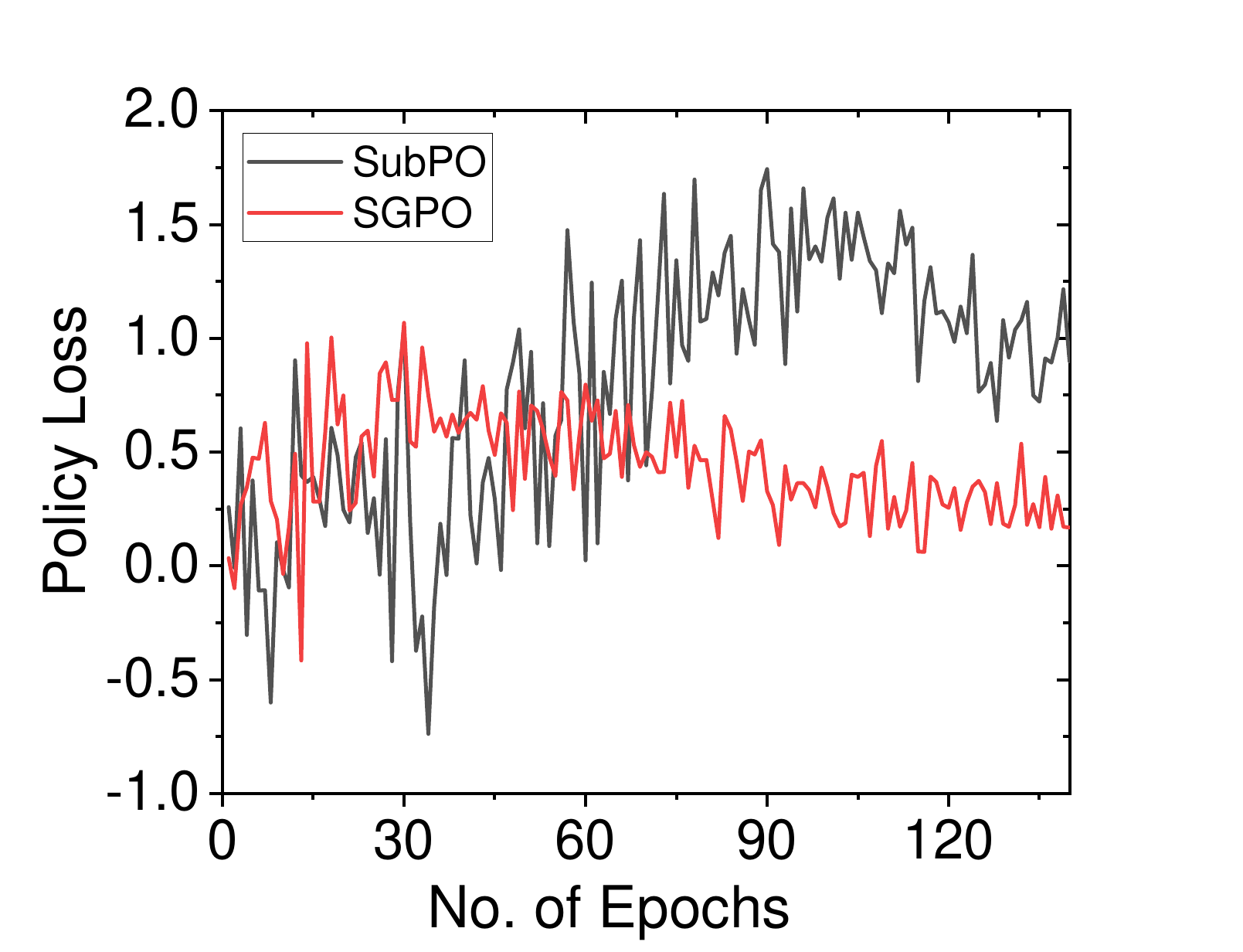}\hspace{-1em} & 
\includegraphics[scale = 0.16]{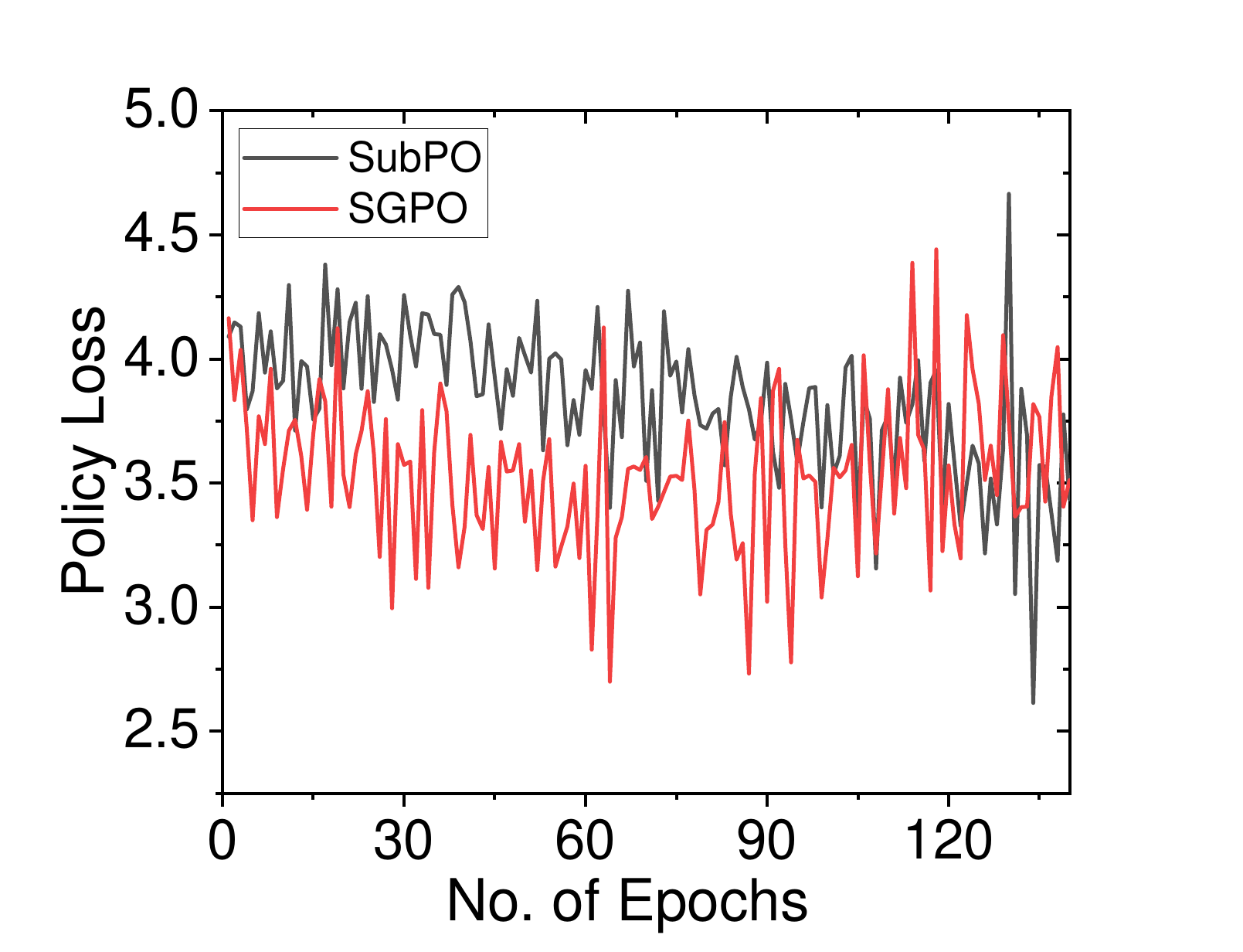}\hspace{-1em}  & 
\includegraphics[scale = 0.16]{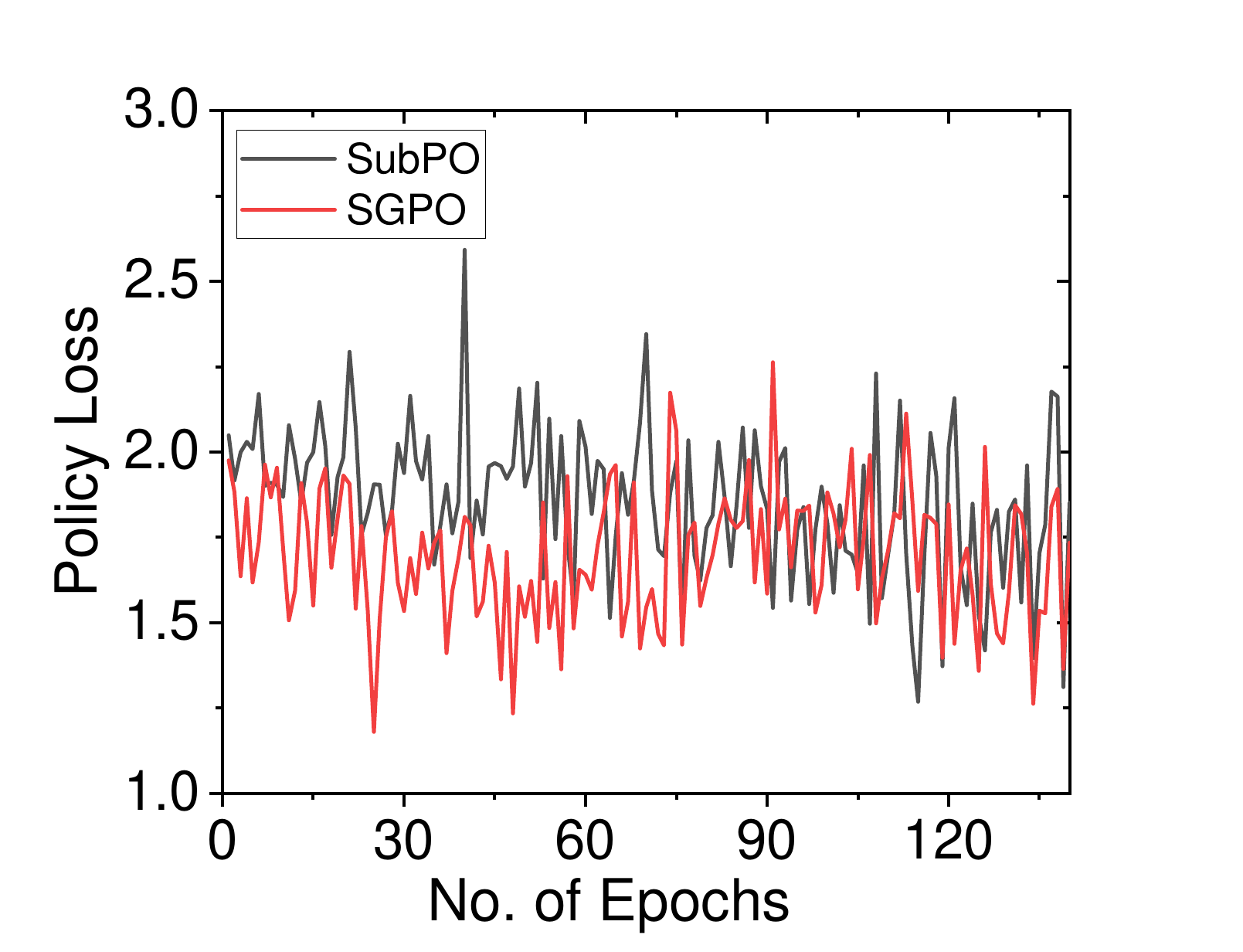} \\
\tiny{(g) Epochs Vs. Policy Loss} & \tiny{(h) Epochs Vs. Policy Loss} & \tiny{(i) Epochs Vs. Policy Loss} \\
\includegraphics[scale = 0.16]{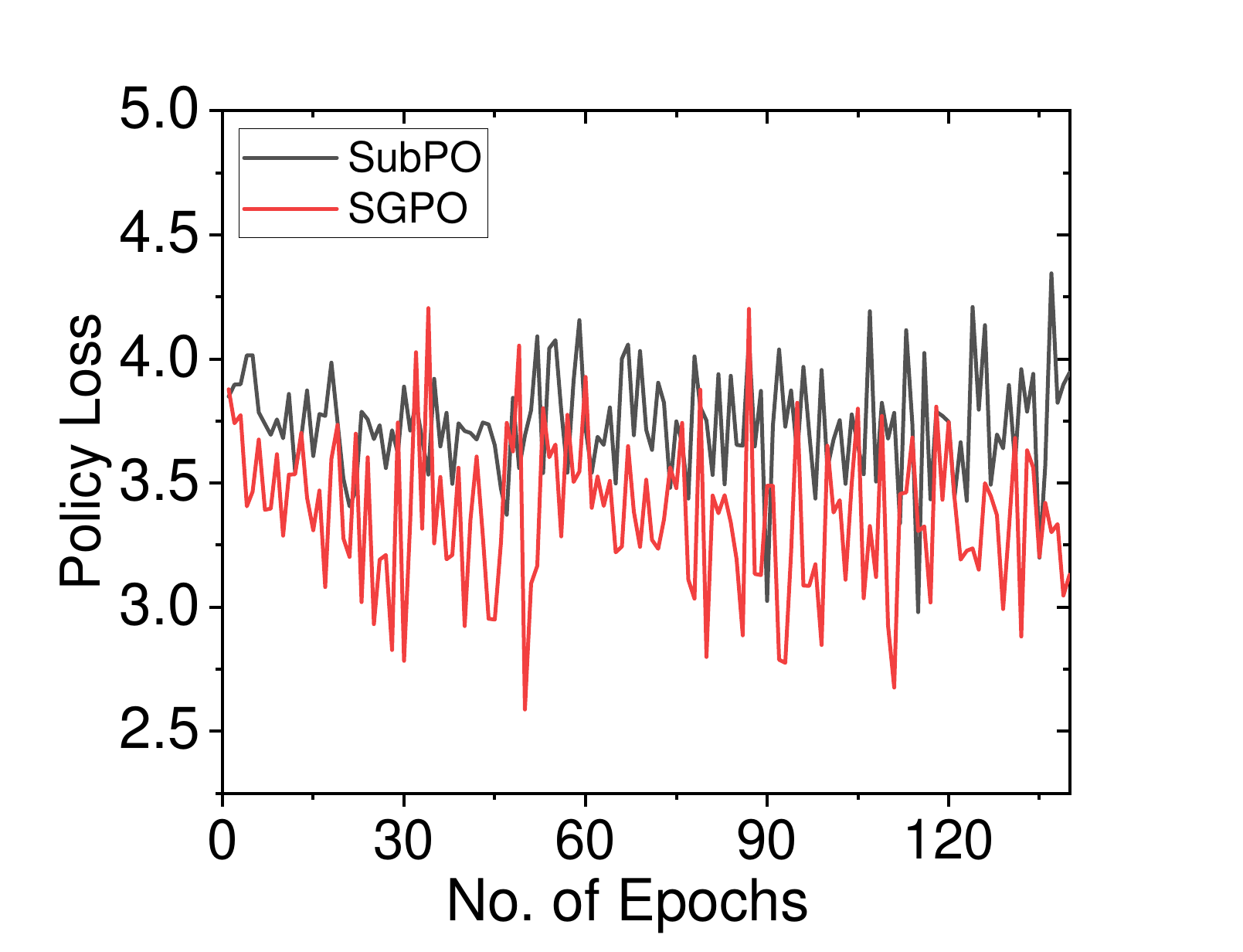}\hspace{-1em} & 
\includegraphics[scale = 0.16]{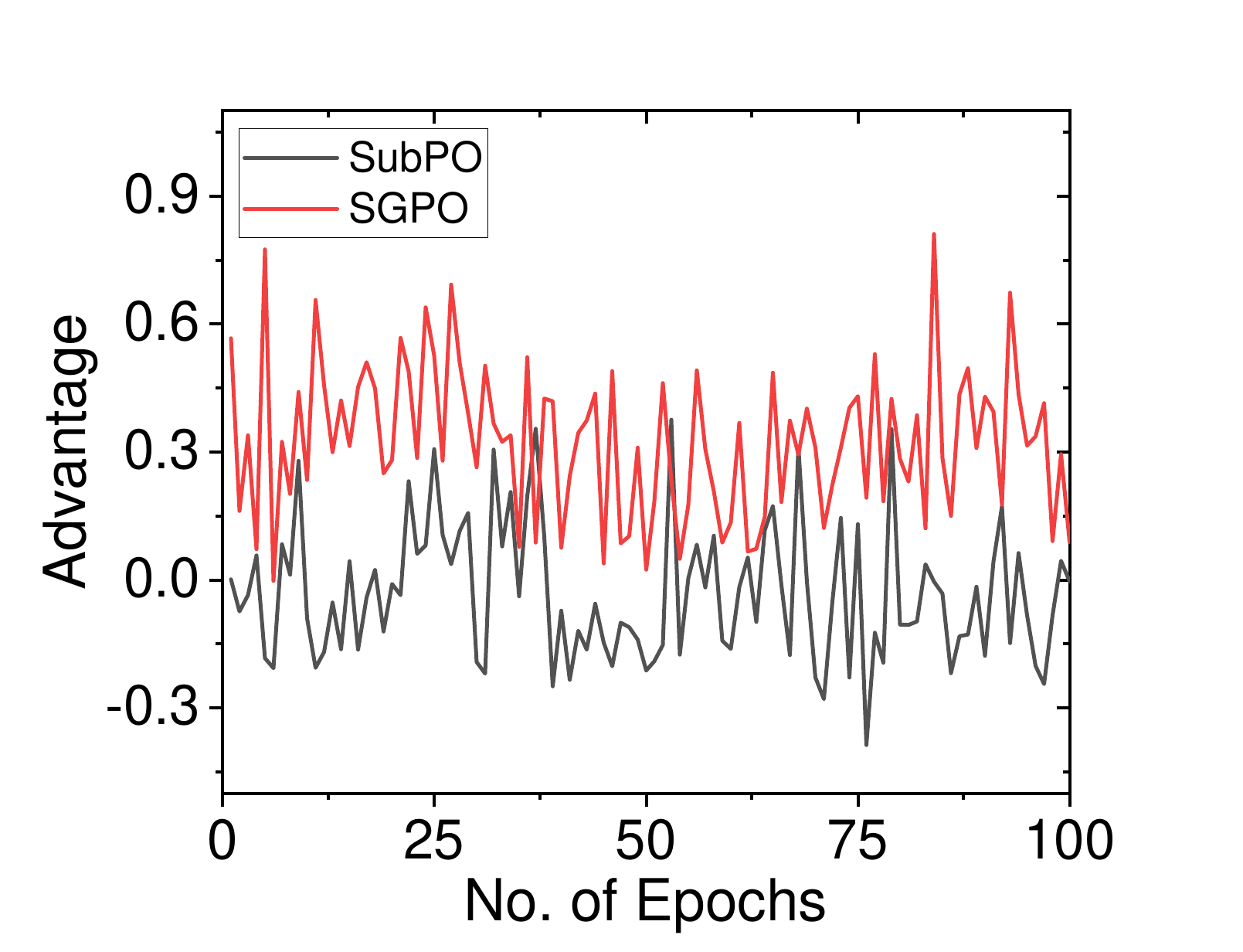}\hspace{-1em} & 
\includegraphics[scale = 0.16]{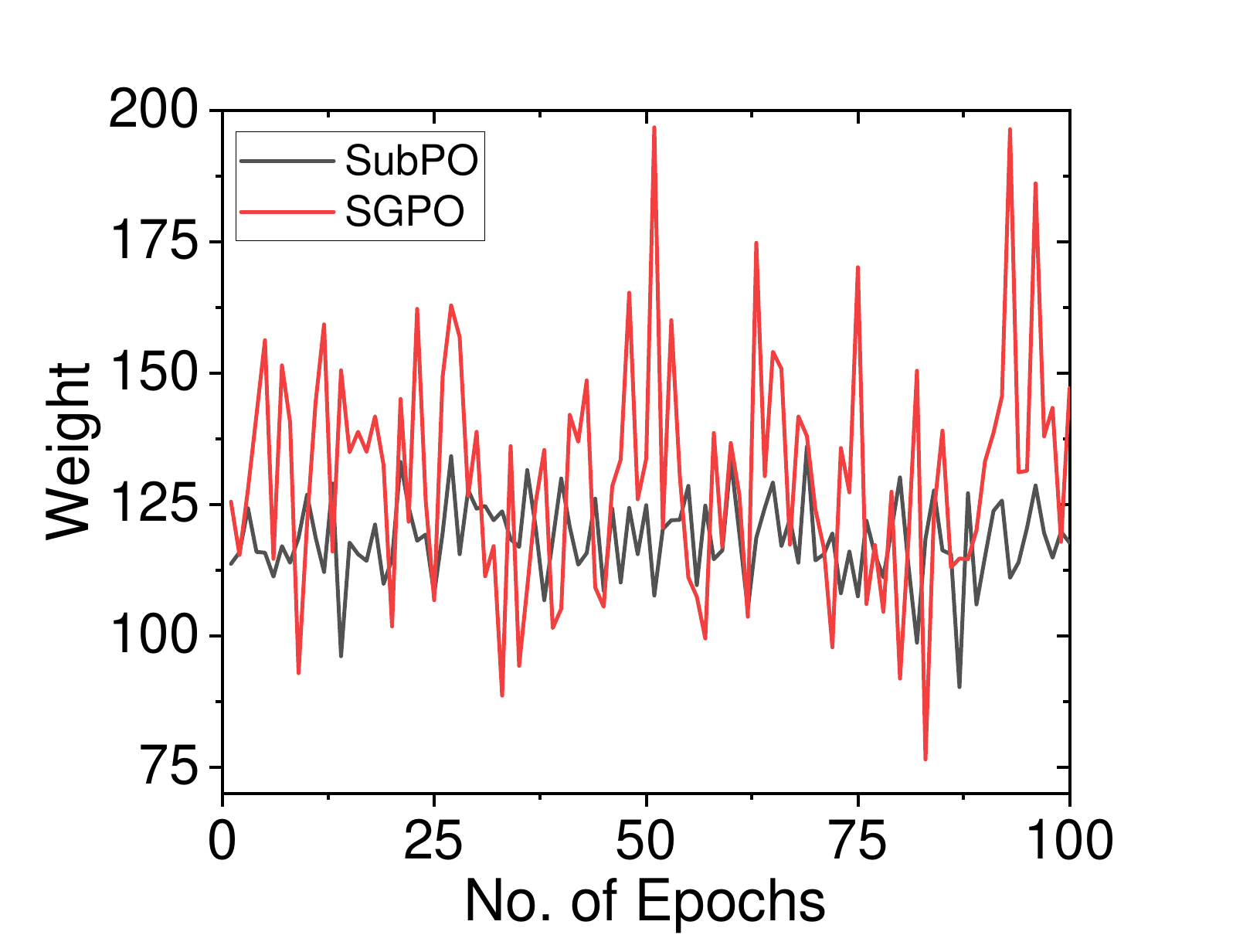} \\
\tiny{(j) Epochs Vs. Policy Loss} & \tiny{(k) Epochs Vs. Advantage} & \tiny{$(\ell)$ Epochs Vs. Weights} \\ 
\end{tabular}
\caption{ Varying Epochs with Steps, Coverage, Critic Loss, Policy Loss, Advantage, Weight in Car racing $(a,c,e)$, Mujoco Ant $(b,d,f,k)$, Graph Based-M $(g, \ell)$, Graph Based-SRL $(h)$, Entropy Based-M $(i)$, Entropy Based-SRL $(j)$}
\label{Fig:Figure}
\end{figure*}
\vspace{-0.25cm}

\paragraph{\textbf{Car Racing Environment.}} In Figure \ref{Fig:Figure}(a), we illustrate the variation in the number of horizon steps required to complete one lap of the racing track during the test time with the number of epochs used for training. The test was carried out with 100 agents, and the mean number of steps required to complete the race was recorded. We observe that the minimum number of steps required to complete one lap first rises and then falls as the number of training epochs increases, indicating improved training. The model trained using SGPO requires fewer steps compared to the model trained with SubPO. We show the variation in policy loss with the number of epochs during training, as shown in Figure \ref{Fig:Figure}(e). We observe that while training with the SubPO algorithm, the policy loss oscillates within a wider range of high values. In contrast, in the case of SGPO, the loss remains in a relatively lower range, indicating greater stability during training. Next, Figure \ref{Fig:Figure}(c) shows the variation in critic loss with the number of training epochs. The loss curve for SubPO exhibits large fluctuations, whereas the loss curve for SGPO is comparatively smoother, indicating better learning and convergence.

\paragraph{\textbf{Mujoco Ant Environment.}}
Figure \ref{Fig:Figure}(b) shows the variation in total coverage by an agent with the number of training epochs. The test was carried out with 100 agents, and their mean coverage was recorded. For the model trained on 1000 epochs, we observe that SubPO achieves coverage of 2612 units, whereas SGPO achieves a slightly higher coverage of 2636 units. Next, in Figure \ref{Fig:Figure}(d), we illustrate the number of epochs versus the critic loss. The critic loss curve for SubPO exhibits larger oscillations, whereas SGPO demonstrates better convergence. Although both loss curves exhibit fluctuations, the variance in the SubPO loss curve is significantly higher than that of SGPO, indicating less stability during training. The number of epochs versus advantage in this environment is presented in Figure \ref{Fig:Figure}(k), which compares the advantage values of SubPO and SGPO over 100 epochs in a test environment. Advantage, a key metric in RL, represents the relative improvement of an action over the expected value baseline. The figure shows that SGPO exhibits higher and more fluctuating advantage values compared to SubPO, indicating potentially more exploratory behavior.

\paragraph{\textbf{Graph Based Environment.}}
In the graph-based environment, we vary the number of epochs for two different modes: M and SRL. To analyze the learning and performance of the agent network, we observe that in mode M, the loss curves for both SubPO and SGPO exhibit oscillations, but the one for SGPO shows better convergence, as shown in Figure \ref{Fig:Figure}(g). In the case of SRL mode, both loss curves have spikes and dips, but the loss for SGPO remains at relatively lower values, as shown in Figure \ref{Fig:Figure}(h). Next, we have presented the weight value in M-mode with the varying number of epochs in Figure \ref{Fig:Figure}$(\ell)$. Here, the agent's goal is to maximize the weight. We observe that the weight curve for SubPO remains relatively stable with minor fluctuations, whereas SGPO exhibits larger oscillations with higher peaks. This suggests that SGPO explores a broader range of weight values, potentially leading to higher maximum weight values.

\paragraph{\textbf{Entropy Based Environment.}} In the entropy-based environment, we observe the variation of policy loss with the number of training epochs. In the case of M mode, the agent focuses on minimizing entropy, meaning it aims to reduce the uncertainty in its trajectory and converge to more deterministic behaviors. The policy loss curve for SubPO shows higher fluctuations, whereas SGPO exhibits a more stable trend, indicating better convergence, as shown in Figure \ref{Fig:Figure}(i). In SRL mode, the agent aims to maximize entropy, promote exploration, and avoid early convergence with deterministic behaviors. The policy loss curve for SGPO remains at lower values compared to SubPO, suggesting improved stability and adaptability in exploration, as shown in Figure \ref{Fig:Figure}(j).

\begin{table}[h!]
    \centering
        \caption{Comparison of Objective Function Values across Discrete Environments}
    \begin{tabular}{lcc}
        \toprule
        \textbf{Environment} & \multicolumn{2}{c}{\textbf{Objective Function}} \\
        \cmidrule(lr){2-3}
        & \textbf{SubPO} & \textbf{SGPO} \\
        \midrule
        Graph-based(M) & 119.3361648 & 129.792734 \\
        Graph-based(SRL) & 54.54513195 & 77.3658894 \\
        Entropy-based(M) & 171.211965 & 237.8183107 \\
        Entropy-based(SRL) & 379.1436629 & 386.2395683 \\
        \bottomrule
    \end{tabular}
    \label{tab:comparison}
\end{table}

\section{Concluding Remarks} \label{Sec:Conclusion}
In this paper, we have considered a variant of the reinforcement learning problem where the reward function is submodular. We have proposed a pruned submodularity graph-based approach to solve this problem. The proposed SGPO has been analyzed to understand its time and space requirements and performance guarantee. We show that the proposed SGPO will achieve a constant factor approximation guarantee under the simplified assumptions. We have performed many experiments with continuous and discrete action spaces and observed that SGPO leads to more reward compared to the existing method. Now, this study can be extended to the broad class of submodular objectives, which are relevant to many practical, real-world problems.
\bibliographystyle{splncs04}
\bibliography{Paper}

\end{document}